\documentclass{article}

\usepackage{iclr2027_conference,times}
\usepackage{microtype}
\usepackage{amsmath,amssymb,amsthm,mathtools}
\usepackage{booktabs}
\usepackage{array}
\usepackage{wrapfig}
\usepackage{enumitem}
\usepackage{algorithm}
\usepackage{algpseudocode}
\makeatletter
\renewcommand{\theALG@line}{\thealgorithm.\arabic{ALG@line}}
\makeatother
\usepackage{xcolor}
\usepackage{tikz}
\usetikzlibrary{arrows.meta,backgrounds,calc,fit,positioning}
\usepackage{hyperref}
\usepackage[nameinlink,capitalize,noabbrev]{cleveref}
\makeatletter
\providecommand*{\theHALG@line}{}
\renewcommand*{\theHALG@line}{\thealgorithm.\arabic{ALG@line}}
\makeatother

\definecolor{linkblue}{RGB}{32,83,164}
\definecolor{sdblue}{RGB}{54,105,179}
\definecolor{sdorange}{RGB}{230,126,34}
\definecolor{sdgreen}{RGB}{48,145,106}
\definecolor{sdgray}{RGB}{145,145,145}
\hypersetup{
  colorlinks=true,
  linkcolor=linkblue,
  citecolor=linkblue,
  urlcolor=linkblue
}

\newtheorem{theorem}{Theorem}
\newtheorem{proposition}[theorem]{Proposition}
\newtheorem{corollary}[theorem]{Corollary}

\theoremstyle{definition}

\theoremstyle{remark}

\crefname{assumption}{assumption}{assumptions}
\Crefname{assumption}{Assumption}{Assumptions}
\crefname{algorithm}{algorithm}{algorithms}
\Crefname{algorithm}{Algorithm}{Algorithms}

\algrenewcommand\algorithmicrequire{\textbf{Input:}}
\algrenewcommand\algorithmicensure{\textbf{Output:}}

\newcommand{\SourceDist}{P}
\newcommand{\TargetDist}{Q}

\newcommand{\Ind}{\mathbf{1}}

\title{When Does a Skill Add Value?\\
Task-Conditional Gain Prediction\\
for Selective Skill Use}
\author{Anjie Xu$^{1,2,3}$ \quad Zhiyu Zhang$^{7}$ \quad Ruiqing Ding$^{4,5}$ \quad Fengli Xu$^{3,6}$ \quad Leye Wang$^{1,2}$\\
{\normalfont\small $^{1}$Key Lab of High Confidence Software Technologies (Peking University), Ministry of Education, China}\\
{\normalfont\small $^{2}$School of Computer Science, Peking University, Beijing, China}\\
{\normalfont\small $^{3}$Zhongguancun Academy, Beijing, China}\\
{\normalfont\small $^{4}$Key Laboratory of Process Optimization and Intelligent Decision-making, Ministry of Education, China}\\
{\normalfont\small $^{5}$School of Management, Hefei University of Technology, Anhui, China}\\
{\normalfont\small $^{6}$Department of Electronic Engineering, BNRist, Tsinghua University, Beijing, China}\\
{\normalfont\small $^{7}$Yuanpei College, Peking University, Beijing, China}}

\hypersetup{
  pdftitle={When Does a Skill Add Value? Task-Conditional Gain Prediction for Selective Skill Use},
  pdfauthor={Anjie Xu, Zhiyu Zhang, Ruiqing Ding, Fengli Xu, Leye Wang},
  pdfsubject={Task-conditional skill-gain prediction, selective skill use, and full-inventory benchmark evaluation}
}

\iclrfinalcopy

\begin{document}
\raggedbottom
\maketitle
\fancyhead{}
\renewcommand{\headrulewidth}{0pt}

\begin{abstract}
Agent skills are expected to improve task performance. Yet we find that they
often provide no benefit, and can even hurt performance while incurring
additional token costs. Can we predict whether a skill will help before the
agent acts? We introduce \textsc{SkillDelta}, a framework for predicting
task-conditional skill gains from paired executions of the same agent with
and without the skill. A local predictor transfers these historical gains
to new tasks without retraining the agent. Under explicit transfer assumptions,
our analysis links support coverage, representation mismatch, and execution noise to prediction error
and decision regret. Across five benchmarks and three target agents, paired
history improves observed-gain ranking over skill-assisted outcomes alone in
12 of 15 settings. At matched expected skill-use rates, \textsc{SkillDelta}
improves success over random activation in all 15 settings, with an average
absolute gain of 4.3\%. Most of this advantage comes from allocation across
task groups. Evidence for additional within-group selection value is strongest
on ToolQA and weaker elsewhere. Code is available at
\href{https://github.com/TankTechnology/skilldelta}{this public repository}.

\end{abstract}

\section{Introduction}
\label{sec:introduction}

Agent skills promise to improve task execution through reusable procedures
without retraining the agent \citep{zhang2025agentskills}. Yet relevant
instructions can duplicate existing ability, reduce success, or add tokens
without benefit. A tool call invokes an operation; a skill supplies a workflow
that can guide several tools. Models can learn tool use during training
\citep{schick2023toolformer,eleti2023functioncalling}, whereas users can supply
skills afterward, with descriptions that may overstate or misrepresent their
benefits. \emph{Can we predict whether a supplied skill will improve the agent's
chance of success on its current task, before execution?}

Existing work measures skill utility
\citep{li2026skillsbench,han2026sweskillsbench,liu2026wildskills}, improves
retrieval \citep{zheng2026skillrouter}, and optimizes skill content
\citep{yang2026skillopt}. SRA-Bench further shows that agents struggle to judge
when skills are needed \citep{su2026sra}. However, finding a relevant skill or
succeeding with it does not establish benefit: the agent may already solve
the task without it. Model routers compare answering models
\citep{ong2024routellm,somerstep2025carrot}; our question holds the agent fixed
and asks how its success changes when given a supplied skill. This requires
predicting \emph{incremental gain} before either outcome is observed.

Predicting benefit requires distinguishing an intervention's effect from its
outcome \citep{shalit2017estimating}. Agent tasks let us measure this difference
through \emph{paired execution history}: running the same agent under use and
skip conditions, with repetition improving precision. Deployment raises a
generalization question: when can gains measured on historical tasks predict
benefit on a new task? Topic similarity alone is insufficient; useful history
must cover the target, and the representation must capture similarity in skill
benefit.

We develop \textsc{SkillDelta}, a local predictor that aggregates paired gains
from eligible historical tasks without retraining the agent. Skill or family
information restricts the support, and fixed task embeddings identify neighbors.
Building on local regression
analysis \citep{kpotufe2011knn}, we connect support coverage, representation
mismatch, and execution noise to prediction error and decision loss under
explicit transfer assumptions. The analysis explains why expanding a
history bank and repeating its executions address different sources of error.
An accompanying DeepSeek Harness plugin gates skill injection before the first
model step.

We test both the value of paired evidence and the decisions it supports
across five benchmarks and three agents. Paired history improves observed-gain
ranking over skill-assisted outcomes alone in 12 of 15 settings. At matched
expected skill-use rates, selective use yields an average absolute success-rate
gain of 4.3\% over random activation (\cref{fig:matched-rate-three-stack}).
This advantage combines allocation across task groups with selection within
groups; the clearest within-group benefits occur on ToolQA
(\cref{fig:selection-components}). Evidence-scaling diagnostics examine how
support, neighborhood size, and repetitions affect prediction error.

\noindent\textbf{Contributions.}
\begin{enumerate}[leftmargin=*,itemsep=0pt,parsep=0pt,topsep=2pt]
  \item We formulate pre-execution skill-gain prediction, distinguishing
    incremental benefit from skill-assisted success, and characterize how
    coverage, representation mismatch, and execution noise constrain
    transfer from paired history.
  \item We introduce \textsc{SkillDelta}, a local gain estimator that guides
    skill use before execution without retraining the agent, and provide a
    DeepSeek Harness plugin for integration.
  \item Paired history enables higher success than random activation at
    matched expected skill-use rates across five benchmarks and three agents,
    chiefly through better allocation across task groups.
\end{enumerate}

\noindent Our implementation and data-collection code, including a \textsc{SkillDelta}
plugin for DeepSeek Harness, are available in the
\href{https://github.com/TankTechnology/skilldelta}{public repository}.

\section{Problem Definition: Task-Conditional Skill Gain}
\label{sec:problem-definition}

\begin{figure}[t]
  \centering
  \includegraphics[width=\linewidth]{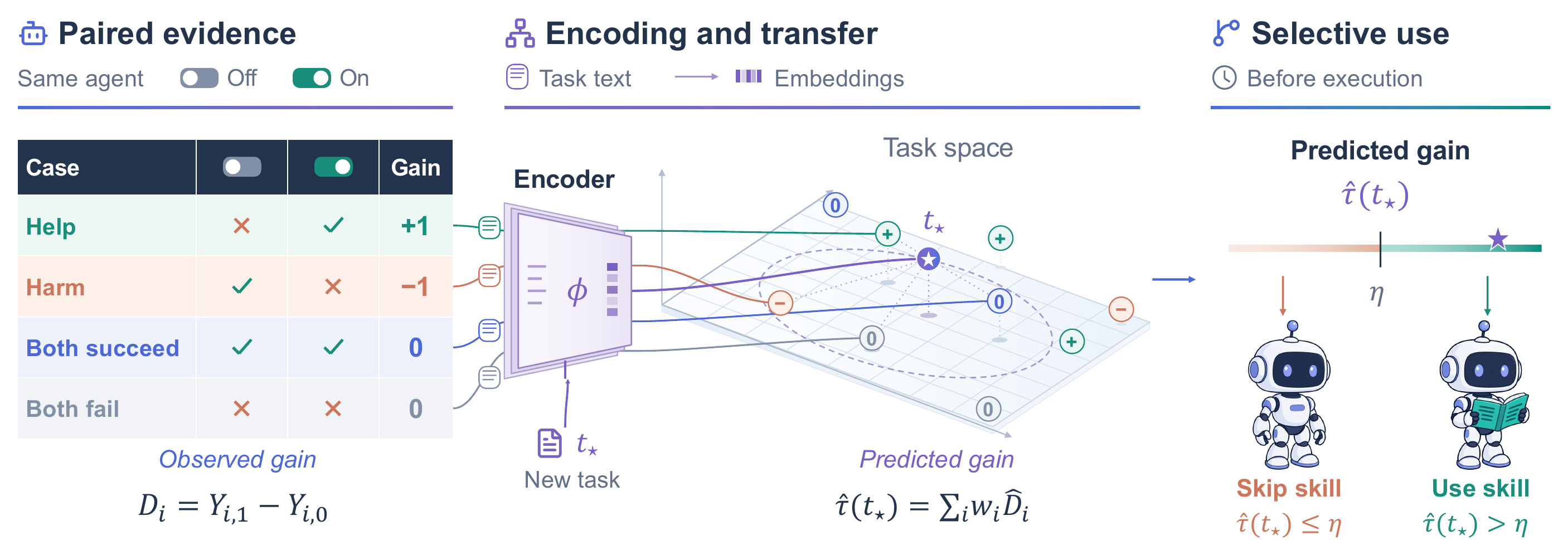}
  \caption{\textsc{SkillDelta} predicts skill gains before execution.
  Paired outcomes provide observed gains, which are transferred through
  nearby task embeddings to predict gain on a new task. Thresholding this
  prediction determines whether to use the skill. Only task text is encoded;
  the geometry is schematic.}
  \label{fig:skilldelta-pipeline}
\end{figure}

\subsection{Expected gain and paired observations}

Let $t\in\mathcal T$ be a task, $s=s(t)$ its supplied applicable skill or fixed
skill bundle, and $s_0$ the no-skill condition. Both conditions use the same
agent and execution protocol, differing only in skill injection. For a binary
success outcome $Y(t)$, define
\begin{equation}
  \begin{aligned}
    \mu_1(t,s)   &= \mathbb E[Y(t)\mid t,s]
                    = \Pr(Y(t)=1\mid t,s),\\
    \mu_0(t,s_0) &= \mathbb E[Y(t)\mid t,s_0]
                    = \Pr(Y(t)=1\mid t,s_0),\\
    \tau(t,s)    &= \mu_1(t,s)-\mu_0(t,s_0).
  \end{aligned}
  \label{eq:conditional-effect}
\end{equation}
The subscripts label execution conditions. The gain $\tau(t,s)$ is the change
in the agent's success probability when supplied with the skill.
When the supplied skill is fixed by context, we write
$\mu_1(t)\equiv\mu_1(t,s)$, $\mu_0(t)\equiv\mu_0(t,s_0)$,
$\tau(t)\equiv\tau(t,s)$, and $\widehat\pi(t)\equiv\widehat\pi(t,s)$.

For example, two tasks may each have 80\% success with the skill. If their
no-skill success probabilities are 80\% and 40\%, their gains are zero and
40\%, respectively. Predicting skill-assisted success alone treats these
tasks alike, although only the second benefits from the skill.

For support task $t_i$ with skill $s_i=s(t_i)$, the success fractions
$\widehat\mu_1(t_i,s_i)$ and $\widehat\mu_0(t_i,s_0)$ from $r_i$ repetitions
per condition estimate the probabilities in \cref{eq:conditional-effect}.
Their difference is the observed paired gain:
\begin{equation}
  \widehat D_i=\widehat\mu_1(t_i,s_i)-\widehat\mu_0(t_i,s_0),
  \qquad
  \mathbb E[\widehat D_i\mid t_i,s_i]=\tau(t_i,s_i).
  \label{eq:observed-paired-gain}
\end{equation}
With one paired execution, $\widehat D_i\in\{-1,0,1\}$ records realized harm,
no change, or help. Pairing matches the task across conditions; it does not
require identical execution trajectories. Observed differences combine expected gains with
execution noise. Independent repetitions under a fixed execution
distribution reduce the noise in $\widehat D_i$.

\subsection{Selective use as a downstream application}

The policy $\widehat\pi_\eta$ maps each task to an action---use (1) or skip
(0)---by comparing its predicted gain $\widehat\tau(t,s)$ with threshold $\eta$:
\begin{equation}
 \widehat\pi_{\eta}(t,s)=\Ind\{\widehat\tau(t,s)>\eta\}.
 \label{eq:success-first-gate}
\end{equation}
At $\eta=0$, the policy selects positive predicted gains; a larger threshold
requires a larger predicted benefit.
We evaluate success rate and token usage against Always-off (skip the skill)
and Always-on (use it). \Cref{sec:theory} develops the gain estimator;
\cref{sec:experiments} specifies the evaluation protocol.

\section{Method and Theoretical Analysis}
\label{sec:theory}

\textsc{SkillDelta} averages observed gains from similar historical tasks to
predict whether a supplied skill will help a new task. The target provides
its description and skill identity; execution outcomes come only from
support tasks. We analyze when this transfer is reliable, describe the
predictor, and connect its errors to decision loss. Proofs and extensions
are in \cref{app:proofs}.

\subsection{What can transfer?}

Identical historical outcomes can be consistent with either benefit or harm
on a new task (\cref{app:no-free-lunch}). Predicting its gain therefore requires
a relation between support and target effects. For each eligible pair, we assume

\begin{equation}
  |\tau(t)-\tau(t')|\leq Ld(\phi(t),\phi(t'))+\varepsilon_\phi.
  \label{eq:effect-smoothness}
\end{equation}
Here $\phi$ represents task text and $d$ measures distance.
The residual $\varepsilon_\phi$ accounts for gain differences left unexplained
by the representation distance.

Consider a calculator skill and two questions about the same medical
quantity: one asks for a definition, the other for a unit conversion.
Their topics are similar, but the calculation procedure may help only the
second. Predicting gains requires a representation that captures such
differences in benefit, with $\varepsilon_\phi$ allowing for those it misses.

\subsection{What limits gain prediction?}

For support task $i$, write
$\widehat D_i=r_i^{-1}\sum_{\ell=1}^{r_i}D_{i\ell}$, where
$D_{i\ell}\in[-1,1]$ are independent across tasks and repetitions, with mean
$\tau(t_i)$. Let nonnegative weights sum to one
and let $\widehat\tau(t)=\sum_i w_i(t)\widehat D_i$. Define the weighted radius
$\rho_w(t)=\sum_iw_i(t)d(\phi(t),\phi(t_i))$.

\begin{theorem}[Local gain-prediction bound]
\label{thm:local-bound}
Under \cref{eq:effect-smoothness}, for any fixed target $t$, weights independent
of support outcomes, and $\delta\in(0,1)$, with probability at least $1-\delta$,
\begin{equation}
  |\widehat\tau(t)-\tau(t)|
  \leq
  \underbrace{L\rho_w(t)}_{\textnormal{coverage}}
  +\underbrace{\varepsilon_\phi}_{\textnormal{representation residual}}
  +\underbrace{
    \sqrt{2\log(2/\delta)\sum_i\frac{w_i(t)^2}{r_i}}
  }_{\textnormal{execution noise}}.
  \label{eq:main-error-bound}
\end{equation}
\end{theorem}

The same $\varepsilon_\phi$ appears in both equations because averaging with
weights that sum to one preserves the residual bound.
The bound distinguishes three ways to change the evidence. A larger support
pool can supply closer neighbors. Using more neighbors can reduce noise but
increase transfer bias; repeating executions improves gain precision at fixed
coverage and representation. These distinctions guide the evidence-scaling
diagnostics. The bound's numerical strength at small neighborhoods and the
additional conditions for future-task guarantees are discussed in
\cref{app:local-bound-proof,app:fresh-target-generalization}.

\subsection{Local prediction from paired evidence}

\Cref{fig:skilldelta-pipeline} illustrates the default signed-gain predictor.

\paragraph{Measure paired gains.}
The left table separates help, harm, success in both conditions, and failure
in both. Both unchanged outcomes have zero gain; the observed differences
$\widehat D_i$ provide the support labels.

\paragraph{Transfer gains from nearby tasks.}
In the middle panel, the support rule uses skill or family metadata to select
eligible tasks, for example those sharing group $g(t)$. Question-only vectors
$\phi(t)$ determine similarity within that scope. These steps play different
roles: the support rule specifies which historical skill interventions can
inform the target, while the encoder locates similar questions among them.
Sharing a skill or family makes tasks eligible; it does not assign them equal
gains. Neither observed successes nor gains enter the encoder or determine
the neighbor weights.
We select up to $k$ nearest
tasks, shown around the star, and weight their gains by clipped cosine
similarity:
\begin{equation}
  \widehat\tau(t,s)=\sum_{i\in N_k(t,s)}w_i(t,s)\widehat D_i,
  \qquad
  w_i(t,s)=
  \frac{[\cos(\phi(t),\phi(t_i))]_+}
       {\sum_{j\in N_k(t,s)}[\cos(\phi(t),\phi(t_j))]_+}.
  \label{eq:skilldelta-score}
\end{equation}
If all clipped similarities are zero, we use uniform weights. With shared
weights, this estimator also equals
$\widehat\mu_{1,w}(t)-\widehat\mu_{0,w}(t)$, where
$\widehat\mu_{a,w}(t)=\sum_iw_i(t)\widehat\mu_a(t_i)$.

For illustration, neighbors with gains $+1$, $0$, and $-1$ and weights
$0.5$, $0.3$, and $0.2$ give $\widehat\tau=0.3$. Both helpful and harmful
evidence affect the estimate. Zero-gain neighbors retain their weight,
bringing the prediction toward zero rather than disappearing from the
calculation.

\paragraph{Predict gain and choose an action.}
The right panel applies \cref{eq:success-first-gate} to the signed estimate.
Let $\mathcal I_{-t}$ index historical support tasks excluding target $t$.
\Cref{alg:skilldelta} returns the estimate and action; with no eligible support,
it skips the skill and reports zero as a fallback score.
All main experiments use this signed-gain rule; shared support and threshold
settings appear in \cref{sec:experiments}.

\begin{algorithm}[H]
  \caption{Default signed-gain prediction and selective skill use}
  \label{alg:skilldelta}
  \small
  \begin{algorithmic}[1]
    \Require Paired outcomes $Y_{i,a,\ell}$ for support tasks
      $(t_i,s_i)_{i\in\mathcal I_{-t}}$, target $(t,s)$, and representation $\phi$
    \Require Scope $\mathcal I(t,s)$, neighborhood size $k$, and threshold $\eta$
    \State $\mathcal I\gets\mathcal I(t,s)\cap\mathcal I_{-t}$
    \State \textbf{if} $\mathcal I=\varnothing$ \textbf{then return} $(0,0)$
    \State $\widehat D_i\gets r_i^{-1}\sum_{\ell=1}^{r_i}
      (Y_{i,1,\ell}-Y_{i,0,\ell})$ for $i\in\mathcal I$
    \State $N_k\gets$ indices of the $\min(k,|\mathcal I|)$ largest
      $\cos(\phi(t),\phi(t_i))$, $i\in\mathcal I$
    \State $q_i\gets[\cos(\phi(t),\phi(t_i))]_+$ for $i\in N_k$
    \State $w_i\gets q_i/\sum_{j\in N_k}q_j$ if the denominator is positive;
      otherwise $w_i\gets1/|N_k|$
    \State $\widehat\tau(t,s)\gets\sum_{i\in N_k}w_i\widehat D_i$;
      $\widehat\pi(t,s)\gets\Ind\{\widehat\tau(t,s)>\eta\}$
    \State \Return $(\widehat\tau(t,s),\widehat\pi(t,s))$
  \end{algorithmic}
\end{algorithm}

The historical bank stores task descriptions, skill or family identities,
fixed question vectors, and paired outcomes. Each prediction can be traced to
its retrieved records and weights; new records can be added without refitting
model parameters. At deployment, the upstream system supplies a new task and
candidate skill. \textsc{SkillDelta} encodes the question, retrieves eligible
neighbors, and scores the skill before any target execution, without an
additional generative self-judgment call. The bank can serve subsequent tasks
under the same agent and skill conditions.
We release a DeepSeek Harness plugin that follows this workflow with uniform
neighbor averaging, injecting the supplied skill before the first model step
when predicted gain exceeds the threshold.

\subsection{Decision accuracy and loss}

Let $\widehat a_\eta(t)=\widehat\pi_\eta(t)$ be the action chosen by
\cref{eq:success-first-gate}, $a_\eta(t)=\Ind\{\tau(t)>\eta\}$ its oracle,
and $B(t)$ the right-hand side of \cref{eq:main-error-bound}.

\begin{corollary}[Threshold decision]
\label{cor:threshold}
On the event in \cref{thm:local-bound}, $\widehat a_\eta(t)=a_\eta(t)$ whenever
$|\tau(t)-\eta|>B(t)$.
\end{corollary}

The margin $|\tau(t)-\eta|$ determines how much error a decision can tolerate.
To measure the cost when a decision is wrong, define threshold-relative regret as
\begin{equation}
  R_\eta(\widehat a,t)
  = |\tau(t)-\eta|\,\Ind\{\widehat a_\eta(t)\neq a_\eta(t)\}.
  \label{eq:decision-regret}
\end{equation}
At $\eta=0$, this is the success probability lost relative to choosing with
knowledge of the true gain.

\begin{proposition}[Routing regret is controlled by gain error]
\label{prop:routing-regret}
For every target task and threshold $\eta$,
\begin{equation}
  R_\eta(\widehat a,t) \leq |\widehat\tau(t)-\tau(t)|.
  \label{eq:routing-regret-bound}
\end{equation}
Consequently, under the event in \cref{thm:local-bound},
$R_\eta(\widehat a,t)\leq B(t)$.
\end{proposition}

For example, at a zero threshold, skipping a skill that raises success from
50\% to 51\% loses less than skipping one that raises it from 50\% to 70\%.
AUROC evaluates ranking by observed gain sign; regret also depends on gain
magnitude. This motivates evaluating prediction quality alongside policy
success and token usage. Proofs and extensions appear in
\cref{app:decision-layer,app:fresh-target-generalization}.

Matched-use-rate comparison isolates the value of this task-level choice.
Relative to always skipping, a policy's expected success gain on a fixed
collection is the sum of the gains on selected tasks divided by the total
number of tasks. A random policy with the same expected use rate assigns
skill uses independently of those gains. A success advantage over that
reference therefore measures the value of selecting the recipients, beyond
simply using the skill more often.

\section{Experiment and Results}
\label{sec:experiments}

\subsection{Benchmarks and target-agent stacks}

\begin{wraptable}{L}{0.40\textwidth}
  \centering
  \caption{Task and skill inventory.}
  \label{tab:benchmark-design}
  \small
  \setlength{\tabcolsep}{6pt}
  \begin{tabular}{@{}lrr@{}}
    \toprule
    Benchmark & Tasks & Skills \\
    \midrule
    ToolQA & 1,430 & 14 \\
    MedCalc-Bench & 1,100 & 55 \\
    BigCodeBench & 1,136 & 139 \\
    LogicBench & 760 & 19 \\
    SpreadsheetBench & 399 & 1$^{\ddagger}$ \\
    \bottomrule
  \end{tabular}
  \par\vspace{0.25\baselineskip}
  \raggedright\scriptsize $^{\ddagger}$External skill source, not from the
  SRA-Bench skill corpus.
\end{wraptable}

We evaluate four SRA-Bench domains \citep{su2026sra} and SpreadsheetBench
(\cref{tab:benchmark-design}) under three fixed target-agent stacks:
\texttt{Qwen-Turbo}, \texttt{GLM-5.3-Flash}, and \texttt{DeepSeek-V4-Flash}.
Each task is executed with and without its supplied skill, using temperature
zero, frozen artifacts, and benchmark-provided evaluators.

The skill-enabled condition uses the supplied skill for ToolQA and LogicBench,
a calculator skill for MedCalc-Bench, all gold skills for BigCodeBench, and an
externally optimized spreadsheet skill with a pinned evaluation harness.
The main Qwen-Turbo and GLM-5.3-Flash results use one paired execution per
task, except LogicBench, which uses three-repeat means. DeepSeek-V4-Flash
uses three-repeat means on all five benchmarks. RQ3 uses repeated-execution
panels to study support-gain precision.
Execution profiles and token ledgers are frozen separately for each stack;
outcomes are never pooled across stacks. Protocol details, artifact audits,
and task--skill structure are reported in
\cref{app:empirical-details,app:task-skill-structure}.

\subsection{Within-inventory leave-one-task-out evaluation}

We retrospectively evaluate the completed inventories using leave-one-task-out
prediction: both outcomes of each target are excluded from its support.
All three agents share question-only \texttt{text-embedding-3-small} vectors
(1,536 dimensions), L2-normalized before cosine similarity. On every benchmark,
the predictor averages signed gains from up to $k=6$ same-family neighbors,
using nonnegative cosine weights, and enables the skill when the score is
strictly positive. If fewer than six neighbors exist, it uses all available
neighbors; empty support yields score zero and Skip skill. This common rule
removes benchmark-specific thresholds and aggregation choices
(\cref{app:score-protocol,app:configuration-origin}).
Here, ``family'' denotes the recorded task group: supplied-skill families,
BigCodeBench focal groups, or SpreadsheetBench task types
(\cref{app:task-skill-structure}).

Prospective panels test new targets under their historical frozen settings
(\cref{app:generalization-stress}).
After fixing each action,
we take the recorded success and execution tokens from the skill-enabled
condition if the action is use, and from the no-skill condition otherwise.
This reconstructs policy outcomes without additional task executions.

Direct self-judgment gives the target agent the question and the same skill
text as the skill-enabled condition, without historical outcomes or support
statistics. One shared prompt requests a binary use/no-use action;
the fixed parser accepts a unique decision label and maps ambiguous or empty
responses to no-use.

An outcome-free relevance gate uses task--skill TF--IDF cosine similarity,
with a median threshold fixed before outcomes are read. A with-skill-only
predictor averages skill-enabled outcomes over the same neighbors and weights,
isolating the value of paired gains. TF--IDF kNN and
family-mean controls receive the same paired support as \textsc{SkillDelta}
(\cref{app:empirical-details}).

\paragraph{Evidence-scaling diagnostics.}
RQ3 varies support-pool size, neighborhood size, or retained execution
repetitions and measures gain MAE against fixed empirical targets. The
repetition sweep holds neighborhoods fixed; full settings are in
\cref{app:rq3-settings}.

\subsection{Empirical Results}
\label{sec:results}

\subsubsection{RQ1: Skill-gain structure and prediction}

Both helpful and harmful outcomes occur under the same supplied skill
(\cref{app:task-skill-structure}). This variation, hidden by skill-level
averages, motivates predicting task-level gains.

RQ1 measures AUROC for positive ($\widehat D>0$) versus nonpositive observed
gains on held-out tasks, probing whether representation and support preserve
skill-effect structure (\cref{thm:local-bound}).

\begin{figure*}[ht]
  \centering
  \includegraphics[width=\textwidth]{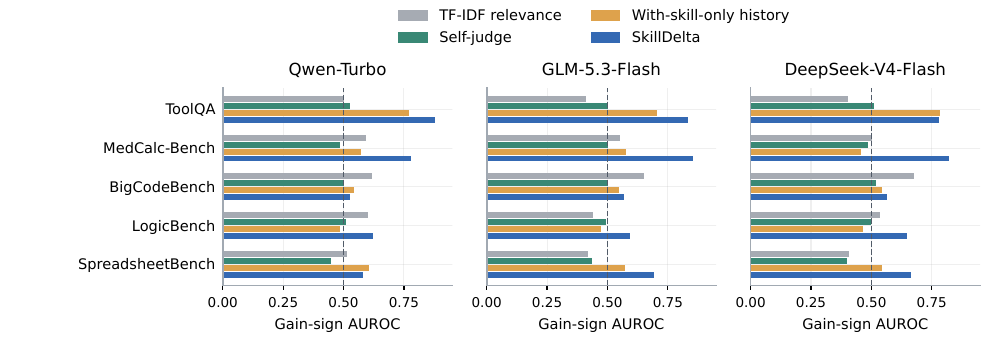}
  \caption{Paired history improves observed-gain ranking in most settings.}
  \label{fig:main-results}
  \begin{minipage}{0.96\linewidth}\scriptsize
  Bars report AUROC for positive versus nonpositive observed gains, comparing
  paired history with relevance, self-judge, and with-skill-only scores.
  Dashed lines denote chance performance. Paired and with-skill-only scores use
  identical target-excluded, same-family $k=6$ neighborhoods and weights.
  \end{minipage}
\end{figure*}

Across the 15 stack--benchmark panels in \cref{fig:main-results}, paired evidence
improves AUROC over the no-outcome reference on 12 panels and over the
with-skill-only reference on 12; it exceeds both on 10. The gains are most
consistent on ToolQA and MedCalc-Bench. BigCodeBench provides the clearest
contrast: on all three stacks, its outcome-free lexical score ranks gains better
than the paired predictor. The with-skill-only predictor exceeds the paired
predictor on Qwen-Turbo BigCodeBench and SpreadsheetBench, and on
DeepSeek-V4-Flash ToolQA.

\paragraph{Can the agent judge for itself?}
Direct self-judgment enables the skill on over 91\% of tasks in 14 of 15
panels; DeepSeek-V4-Flash SpreadsheetBench is the exception at 44.4\%.
Its policy is therefore usually close to Always-on and offers little
gain-sign discrimination. \textsc{SkillDelta}'s continuous scores exceed the
binary self-judge scores in AUROC on all 15 panels. Its decision accuracy and
the judge's routing overhead appear in \cref{tab:no-history-controls}.

\subsubsection{RQ2: Decision value of gain prediction}

At the zero-threshold operating points in \cref{tab:rq2-policy-summary},
\textsc{SkillDelta} improves observed success over Always-off and saves tokens
relative to Always-on on all 15 panels; 14 task-bootstrap intervals for the
success gain exclude zero (\cref{tab:main-intervals}).

At the same expected skill-use rate, \textsc{SkillDelta} achieves higher
success than random activation in all 15 panels, with an average absolute
gain of 4.3\% (\cref{fig:matched-rate-three-stack}).
Matching each recorded task group's use rate separates this advantage into
0.61\% from within-group selection and
3.72\% from between-group allocation (\cref{fig:selection-components}).
The larger allocation term reflects directing the same expected number of
skill uses toward groups with higher average observed gains. This answers
a practical question---which kinds of tasks merit skill use---even when
ranking tasks within each group is difficult.

Finer task selection is clearest on ToolQA: its three-agent mean is 4.41\%,
and all three conditional intervals exclude zero.
Elsewhere, within-group intervals include zero, except for a negative
interval on Qwen-Turbo BigCodeBench.

\begin{table*}[t]
\centering
\caption{Task success and token change under the unified protocol.}
\label{tab:rq2-policy-summary}
\scriptsize
\setlength{\tabcolsep}{4pt}
\begin{tabular}{@{}llrrrrrr@{}}
\toprule
& & \multicolumn{4}{c}{Task success (\%)} & \multicolumn{2}{c}{Token savings (\%)} \\
\cmidrule(lr){3-6}\cmidrule(lr){7-8}
Model & Benchmark & Off & On & Self-judge & \textsc{SkillDelta} & Self-judge & \textsc{SkillDelta} \\
\midrule
 & ToolQA & 28.6 & 44.7 & 44.5 & 44.5 & -5.3 & 10.4 \\
 & MedCalc-Bench & 43.2 & 71.7 & 70.5 & 70.6 & -65.8 & 14.8 \\
Qwen-Turbo & BigCodeBench & 43.7 & 51.4 & 51.4 & 47.4 & -98.2 & 49.4 \\
 & LogicBench & 75.5 & 87.2 & 87.4 & 85.2 & -78.7 & 29.1 \\
 & SpreadsheetBench & 15.5 & 19.8 & 18.3 & 17.8 & -72.3 & 36.2 \\
\midrule
 & ToolQA & 38.7 & 45.0 & 45.1 & 47.8 & -2.6 & 23.8 \\
 & MedCalc-Bench & 78.4 & 83.3 & 83.3 & 88.3 & -27.8 & 41.1 \\
GLM-5.3-Flash & BigCodeBench & 52.0 & 63.3 & 63.2 & 59.8 & -65.4 & 25.4 \\
 & LogicBench & 69.2 & 85.1 & 84.3 & 83.9 & -91.0 & 8.2 \\
 & SpreadsheetBench & 46.1 & 73.4 & 68.4 & 72.4 & -26.0 & 4.8 \\
\midrule
 & ToolQA & 25.5 & 38.3 & 37.9 & 37.9 & -5.8 & 3.0 \\
 & MedCalc-Bench & 74.3 & 89.4 & 88.8 & 90.1 & -66.1 & 12.4 \\
DeepSeek-V4-Flash & BigCodeBench & 49.3 & 61.5 & 61.5 & 57.7 & -85.3 & 26.5 \\
 & LogicBench & 74.4 & 85.1 & 84.0 & 83.6 & -85.6 & 19.5 \\
 & SpreadsheetBench & 46.5 & 61.4 & 49.4 & 59.9 & -28.5 & 8.0 \\
\bottomrule
\end{tabular}
\par\vspace{1mm}
\begin{minipage}{0.98\linewidth}\scriptsize
Token savings are relative to Always-on. \textsc{SkillDelta} uses target-excluded, same-family cosine-weighted $k=6$ signed-gain scores and threshold zero. Its execution costs exclude history collection and encoding; self-judge includes its routing call.
\end{minipage}
\end{table*}

These decisions trade some success for lower execution cost. Across the 15 equally
weighted panels, mean success is 63.12\% versus 64.04\% for Always-on,
with 20.8\% mean token savings (range: $3.0$--$49.4\%$). Success is below Always-on in
12 of the 15 panels. Direct self-judgment increases total tokens by
$2.6$--$98.2\%$ relative to Always-on, including its routing call.

\begin{figure*}[t]
  \centering
  \includegraphics[width=0.99\textwidth]{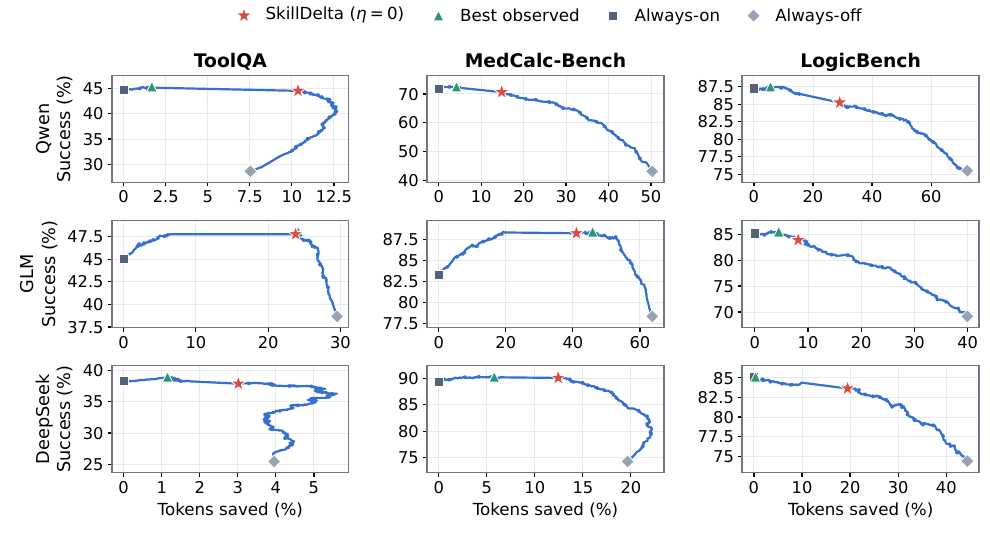}
  \caption{Success--token trade-offs as the decision threshold changes.}
  \label{fig:gain-threshold-tradeoff}
  \par\vspace{1mm}
  \begin{minipage}{0.94\linewidth}\scriptsize
  Fixed within-family, cosine-weighted $k=6$ scores with each target excluded
  from support. Stars mark the default ($\eta=0$); triangles mark the highest
  success found retrospectively in each sweep.
  \end{minipage}
\end{figure*}

\Cref{fig:gain-threshold-tradeoff} varies $\eta$ for ToolQA, MedCalc-Bench, and
LogicBench. Higher thresholds reduce skill use, with success--cost trade-offs
that depend on the benchmark and agent. The default stars reproduce
\cref{tab:rq2-policy-summary}.

Ridge controls use fixed grouped splits for all three agents
(\cref{tab:single-dual-ridge}). Dual Ridge adds an incremental-token head to gain regression;
matched-rule single Ridge uses the same validation search without a cost term.
Cost prediction saves more tokens on several panels, sometimes at lower test success.

\subsubsection{RQ3: What makes transfer work?}

\Cref{fig:rq3-theory-scaling} varies support-pool size $n$, neighborhood size
$k$, and execution repetitions $r$ across three benchmarks and three models.
The rows probe coverage, aggregation, and execution noise in
\cref{eq:main-error-bound}, using held-out gain MAE.

\begin{figure*}[t]
  \centering
  \includegraphics[width=\textwidth]{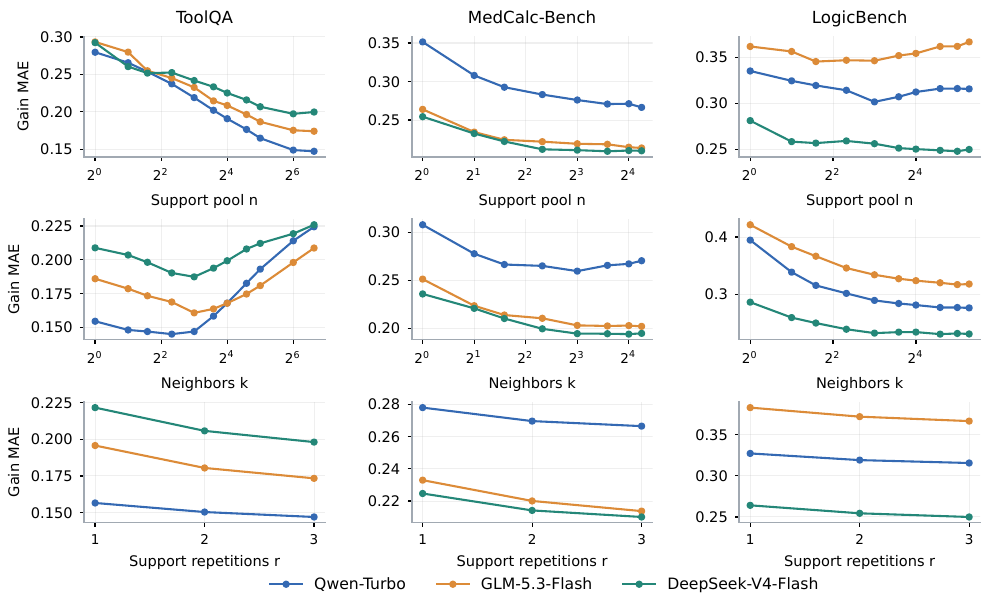}
  \caption{Support, neighborhood size, and repetition affect prediction differently.}
  \label{fig:rq3-theory-scaling}
  \par\vspace{1mm}
  \begin{minipage}{0.96\linewidth}\scriptsize
  Gain MAE (lower is better) against fixed targets while varying support-pool
  size (top), neighbors (middle), or execution repetitions (bottom).
  Support helps unevenly; larger neighborhoods trade locality for pooling;
  repetitions consistently reduce error.
  \end{minipage}
\end{figure*}

\paragraph{Support size: benefits depend on the benchmark.}
Larger within-family pools generally help ToolQA and MedCalc-Bench;
LogicBench varies across agents and need not improve monotonically. Additional
candidates can provide closer neighbors and reduce the coverage term
$L\rho_w$. Their value depends on whether they improve the selected
neighborhood and preserve relevant skill-effect structure.

\paragraph{Neighborhood size: pooling trades noise against locality.}
ToolQA favors intermediate neighborhoods; MedCalc-Bench and LogicBench show
agent-dependent minima or gains from broader pooling. This matches the
aggregation trade-off: adding neighbors
reduces execution noise through larger effective support, but can increase
$\rho_w$ by including tasks with different effects.

\paragraph{Repeated executions: smaller errors with diminishing returns.}
With neighborhoods fixed, increasing repetitions reduces MAE in all nine
panels, with smaller improvements at each step. This is consistent with the
$r^{-1/2}$ noise term: repetition improves support-gain precision while
leaving coverage and representation unchanged.

\section{Related Work}
\label{sec:related}

\paragraph{Skills, retrieval, and adaptive augmentation.}
Skill research also studies testing and coverage
\citep{wang2026skilltester,tan2026skillcoverage}. Tool retrieval
\citep{shi2025toolret} and adaptive augmentation address which external
resources to retrieve and when to use them
\citep{schick2023toolformer,asai2024selfrag,jeong2024adaptiverag}.

\paragraph{Instance-wise performance and heterogeneous effects.}
Algorithm selection and meta-learning predict task-specific performance
\citep{rice1976algorithm,xu2008satzilla,misir2017alors,achille2019task2vec};
performance-prediction methods estimate missing model--dataset entries
\citep{zhao2025crosspred,maiapolo2024tinybenchmarks,wu2026faq}.
Heterogeneous-effect estimation and policy learning connect incremental
outcomes to decisions \citep{shalit2017estimating,
nie2021rlearner,athey2019grf,diemert2021criteo,zhou2023multiaction}.
Two-response estimators predict the mean outcome under each condition and take the difference
\citep{kunzel2019metalearners}; shared local weights recover our signed-gain
estimator. Agent tasks can be repeated under both conditions, allowing us to
study how precisely these differences can be measured and transferred to new tasks.

\paragraph{Transfer and selective decisions.}
Domain adaptation studies coverage and conditional shift
\citep{peyre2019computational,courty2017optimal,redko2017otda,koc2025entanglement},
while selective and conformal prediction formalize uncertainty
\citep{elyaniv2010selective,lei2018distribution}.
We connect the coverage and precision of paired history to skill-use decisions.

\section{Discussion and Limitations}
\label{sec:discussion}

\paragraph{What limits prediction?}
Representation matters: TF--IDF yields higher gain-sign AUROC on three of
the five Qwen-Turbo benchmarks, with BigCodeBench rising from 0.527 to 0.582
under the same paired history (\cref{tab:predictor-baselines}). Semantic
similarity alone need not preserve similarity in skill benefit.
BigCodeBench's sparse groups and heterogeneous skill bundles may also
constrain local evidence (\cref{app:task-skill-structure}). Its natural-shift
panel further shows a success gain over Always-off while
task-level ranking weakens (\cref{app:generalization-stress}).

\paragraph{Deployment cost.}
A practical use case is repeated task execution under stable agent and skill
conditions, where paired evaluation records can be reused for incoming tasks.
Reported token savings cover execution and exclude bank construction, target
encoding, and retrieval. Collecting a bank from scratch adds an investment
whose value depends on reuse; agent or skill changes require rechecking its
evidence. This makes history availability and reuse central to deployment.

\section{Conclusion}

\textsc{SkillDelta} predicts whether a supplied skill improves a fixed agent's
chance of success before execution. Our analysis explains when paired history
supports this transfer. Across five benchmarks and three agents, the resulting selection
improves success over random activation at the same expected skill-use rate.
The gains arise mainly from allocation across task groups, with the clearest
additional within-group value on ToolQA. Paired history can therefore guide
which kinds of tasks receive a skill, even where finer task selection is weak.

\section*{AI Use Statement}
The authors led the research, with AI tools providing assistance in writing,
coding, and analysis. The authors revised and checked the content and take
full responsibility for the final paper.

\section*{Ethics Statement}
This work uses existing agent benchmarks and collects no new human-subject
data. Experiments invoke hosted language-model services; model identities,
token usage, failures, and costs are recorded for audit, while credentials and
private service metadata are excluded from released artifacts. Applications
involving safety-critical instructions require a separate safety assessment
beyond the task-success and token metrics studied here. Deployment trade-offs
are discussed in \cref{sec:discussion}.

\section*{Reproducibility Statement}
The \href{https://github.com/TankTechnology/skilldelta}{public code repository}
provides the core predictor, evaluation utilities, data-collection code, and
DeepSeek Harness plugin. The separate review supplement contains retained
task-level outcomes, embeddings, and scripts for recomputing main metrics,
RQ3 curves, threshold sweeps, selection components, and Ridge controls.
The public release contains code only; experimental protocols and proofs
appear in the appendix.

\bibliographystyle{iclr2027_conference}
\bibliography{references}

\clearpage
\appendix
\section{Proofs and Generalization Guarantees}
\label{app:proofs}

\subsection{Reading guide and notation}
\label{app:theory-guide}

The argument proceeds from the evidence needed to identify a gain, through
prediction error for one task, to the consequences for skill-use decisions.
We then average these results over new tasks. Policy evaluation is a separate
step: it asks what recorded executions tell us about a chosen policy's value.
\Cref{fig:theory-dependencies} summarizes these dependencies.

We use the fixed agent, supplied-skill condition, and notation of
\cref{sec:problem-definition,sec:theory}. Write $\tau(t)$ for $\tau(t,s)$ and
abbreviate the noise factor and error radius in \cref{eq:main-error-bound} as
\begin{equation}
  v_w(t)=\sum_i\frac{w_i(t)^2}{r_i},\qquad
  B_\delta(t)=L\rho_w(t)+\varepsilon_\phi+
  \sqrt{2\log(2/\delta)v_w(t)}.
  \label{eq:appendix-master-radius}
\end{equation}

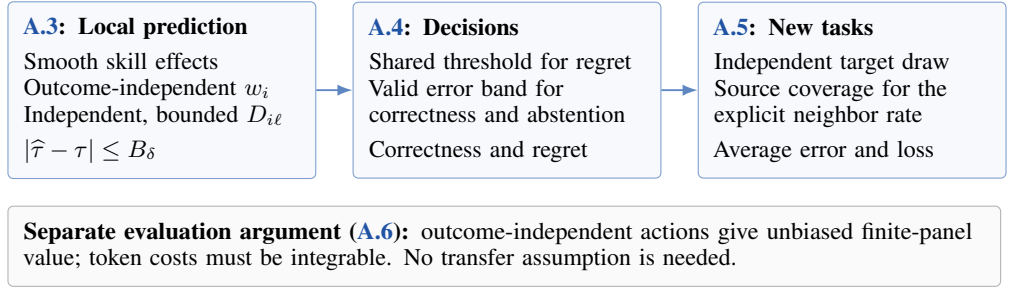
\begin{figure}[H]
  \centering
  \begin{tikzpicture}[
    box/.style={draw=sdblue!65, fill=sdblue!4, rounded corners=2pt,
      align=left, text width=3.65cm, inner sep=6pt,
      minimum height=2.2cm, font=\footnotesize},
    route/.style={-{Latex[length=2mm]}, draw=sdblue!80, line width=0.7pt}
  ]
    \node[box] (local) at (0,0) {
      \textbf{\ref{app:local-bound-proof}: Local prediction}\\[3pt]
      Smooth skill effects\\
      Outcome-independent $w_i$\\
      Independent, bounded $D_{i\ell}$\\[3pt]
      $|\widehat\tau-\tau|\leq B_\delta$};
    \node[box, right=0.48cm of local] (decision) {
      \textbf{\ref{app:decision-layer}: Decisions}\\[3pt]
      Shared threshold for regret\\
      Valid error band for\\
      correctness and abstention\\[3pt]
      Correctness and regret};
    \node[box, right=0.48cm of decision] (fresh) {
      \textbf{\ref{app:fresh-target-generalization}: New tasks}\\[3pt]
      Independent target draw\\
      Source coverage for the\\
      explicit neighbor rate\\[3pt]
      Average error and loss};
    \draw[route] (local) -- (decision);
    \draw[route] (decision) -- (fresh);
    \node[anchor=south west, align=left, text width=13.17cm,
      font=\footnotesize, inner sep=0pt]
      at ([yshift=9pt]local.north west) {
      \textbf{Starting point (\ref{app:no-free-lunch}):}
      unseen-task gain needs a relation between support and target effects.};
    \node[draw=sdgray!70, fill=sdgray!5, rounded corners=2pt,
      anchor=north west, align=left, text width=12.71cm,
      inner sep=6pt, font=\footnotesize]
      at ([yshift=-10pt]local.south west) {
      \textbf{Separate evaluation argument (\ref{app:finite-panel-value}):}
      outcome-independent actions give unbiased finite-panel value;
      token costs must be integrable. No transfer assumption is needed.};
  \end{tikzpicture}
  \caption{How the theory fits together. The main chain turns historical
  evidence into prediction and decision guarantees. Policy evaluation uses
  a separate action--outcome independence argument.}
  \label{fig:theory-dependencies}
\end{figure}

\paragraph{What is random?}
For a fixed-target bound, task identities, representations, repetition counts,
and weights are fixed; probability is over support execution outcomes.
The fresh-task result additionally draws a target $T\sim\TargetDist$.
The neighbor-rate result also draws the support tasks. These are three
different sources of uncertainty, introduced only where needed below.

\subsection{Identification and observed gains}
\label{app:no-free-lunch}

This first result explains why a transfer assumption is needed: collecting
more support outcomes cannot determine an unrelated target's gain.

\begin{proposition}[No free lunch for unseen-task gain]
\label{prop:no-free-lunch}
Let $S=\{t_1,\ldots,t_n\}\subsetneq\mathcal T$ be the executed support tasks
and $t_\star\notin S$. Without a restriction relating $\tau(t_\star)$ to
$\{\tau(t_i)\}_{i=1}^n$, there exist two data-generating processes that agree
on the joint distribution of all observed support executions but satisfy
$\tau^{+}(t_\star)>0$ and $\tau^{-}(t_\star)\leq0$. Consequently, no predictor
can uniformly identify the target gain sign.
\end{proposition}

\begin{proof}
Fix the joint support-outcome distribution. Two extensions can both have
$\mu_0(t_\star)=1/2$, with $\mu_1^+(t_\star)=3/4$ and
$\mu_1^-(t_\star)=1/4$. They give opposite target gains but identical
distributions for every observed support variable. A support-based predictor
therefore has the same distribution in both worlds and cannot be correct
uniformly.
\end{proof}

\paragraph{Observed signs and expected gains.}
If the two conditions independently succeed with probability $1/2$, then
$\tau(t)=0$, yet one paired execution gives
$\Pr(\widehat D=1)=\Pr(\widehat D=-1)=1/4$.
Thus the observed label $\Ind\{\widehat D>0\}$ used for RQ1 AUROC differs
from the expected gain estimated by $\widehat\tau$.

\subsection{Local gain-prediction bound}
\label{app:local-bound-proof}

We prove \cref{thm:local-bound} by separating transfer mismatch from execution
noise. Condition on a nonempty eligible support set, representations,
repetition counts, and outcome-independent weights throughout this proof.
The bound includes the cosine-weighted predictor. With empty support, the
algorithm's zero score is a fallback, not an estimated zero gain.

\begin{proof}[Proof of \cref{thm:local-bound}]
\textit{Step 1: separate the two errors.}
Add and subtract the weighted expected gains:
\begin{align}
  \widehat\tau(t)-\tau(t)
  &=\underbrace{\sum_iw_i(t)\bigl(\widehat D_i-\tau(t_i)\bigr)}_{Z_w(t):\ \text{execution noise}}
  +\underbrace{\sum_iw_i(t)\bigl(\tau(t_i)-\tau(t)\bigr)}_{b_w(t):\ \text{transfer mismatch}}.
  \label{eq:proof-decomp}
\end{align}

\textit{Step 2: bound transfer mismatch.}
By \cref{eq:effect-smoothness}, nonnegative weights, and $\sum_iw_i=1$,
\[
 |b_w(t)|\leq\sum_iw_i(t)|\tau(t_i)-\tau(t)|
 \leq L\rho_w(t)+\varepsilon_\phi.
\]

\textit{Step 3: bound execution noise.}
Write $Z_w(t)$ as a sum of independent centered variables
\(
  \sum_i\sum_{\ell=1}^{r_i}
  (w_i/r_i)(D_{i\ell}-\mathbb E D_{i\ell}).
\)
Each uncentered summand lies in an interval of length $2w_i/r_i$. Hoeffding's
inequality \citep{hoeffding1963probability} therefore gives
\begin{equation}
  \Pr\left(
    |Z_w(t)|\geq x
  \right)
  \leq
  2\exp\!\left(-\frac{x^2}
  {2v_w(t)}\right).
\end{equation}
Setting the right-hand side to $\delta$ gives
$|Z_w(t)|\leq\sqrt{2\log(2/\delta)v_w(t)}$ with probability at least
$1-\delta$. Adding the two bounds proves the result.
\end{proof}

\paragraph{Neighborhood size and family averaging.}
With $k$ equally weighted neighbors and $r$ repetitions, $v_w=1/(kr)$.
Increasing $k$ reduces noise but can enlarge $\rho_w$; family averaging is
useful when effects vary little within the family. At equal repetition counts,
uniform weights minimize the noise term on a fixed support set, while the
total bound also depends on transfer bias.

The number of neighbors can overstate the amount of averaging when a few
weights dominate. Effective support makes this distinction explicit:
\begin{equation}
  n_{\mathrm{eff}}(t)
  =\left(\sum_i w_i(t)^2\right)^{-1},\qquad
  v_w(t)=\frac{1}{r\,n_{\mathrm{eff}}(t)}
  \quad\text{when all }r_i=r.
  \label{eq:effective-support}
\end{equation}
Substitution into \cref{eq:appendix-master-radius} shows how diffuse weights
and repetitions reduce the noise term.

\paragraph{Numerical strength at small neighborhoods.}
For at most $k$ nonzero weights and equal repetition count $r$,
Cauchy--Schwarz gives $v_w\geq1/(kr)$.
At $\delta=0.05$ and the main setting $k=6$, the minimum noise terms are
approximately $1.109$ for $r=1$ and $0.640$ for $r=3$.
Since $|\tau|\leq1$, the first case cannot certify a
zero-threshold decision. Here the bound explains the trade-off; numerical
certification needs more evidence or tighter uncertainty estimates.

\subsection{From gain error to decisions}
\label{app:decision-layer}

The following results translate prediction error into three decision
questions: when the binary action is correct, how much a wrong action costs,
and when an uncertainty band supports a definite action. Regret control is
algebraic; correctness statements use the event that the error band is valid.

\subsubsection{Threshold correctness}
\label{app:threshold-proof}

\begin{proof}[Proof of \cref{cor:threshold}]
On the event
$|\widehat\tau(t)-\tau(t)|\leq B(t)$, suppose
$|\tau(t)-\eta|>B(t)$. If $\tau(t)>\eta$, then
$\widehat\tau(t)\geq\tau(t)-B(t)>\eta$; if $\tau(t)<\eta$, then
$\widehat\tau(t)\leq\tau(t)+B(t)<\eta$. Thus the decisions agree. Taking the
contrapositive, disagreement on the same event implies
$|\tau(t)-\eta|\leq B(t)$.
\end{proof}

\subsubsection{Routing regret}
\label{app:routing-regret-proof}

\begin{proof}[Proof of \cref{prop:routing-regret}]
If $\widehat a_\eta(t)\neq a_\eta(t)$, then $\eta$ lies between
$\widehat\tau(t)$ and $\tau(t)$. Hence
\[
  |\tau(t)-\eta|\leq|\widehat\tau(t)-\tau(t)|.
\]
Multiplying by the disagreement indicator proves
\cref{eq:routing-regret-bound}; if the actions agree, the left-hand side is
zero. On the event in \cref{thm:local-bound}, regret is at most $B_\delta(t)$.
\end{proof}

\subsubsection{Uncertainty bands and selective abstention}
\label{app:uncertainty}

To allow abstention, let
\(
I_\delta(t)=[\widehat\tau(t)-B_\delta(t),
\widehat\tau(t)+B_\delta(t)]
\)
and define a three-way gate that uses the skill when the lower endpoint exceeds
$\eta$, does not use it when the upper endpoint is at most $\eta$, and
abstains otherwise.

\begin{corollary}[Selective correctness from a valid error band]
\label{cor:selective-uncertainty}
On the event
$|\widehat\tau(t)-\tau(t)|\leq B_\delta(t)$, every non-abstaining action agrees
with the threshold oracle. Moreover, abstention implies
$|\tau(t)-\eta|\leq2B_\delta(t)$. The conservative binary rule
\begin{equation}
  \widehat a_{\mathrm{safe}}(t)
  =\Ind\{\widehat\tau(t)-B_\delta(t)>\eta\},
  \label{eq:conservative-gate}
\end{equation}
never invokes the skill when $\tau(t)\leq\eta$ on this event, and its
threshold-relative regret is at most $2B_\delta(t)$.
\end{corollary}

\begin{proof}
If the lower endpoint exceeds $\eta$, then
$\tau(t)\geq\widehat\tau(t)-B_\delta(t)>\eta$; if the upper endpoint is at most
$\eta$, then
$\tau(t)\leq\widehat\tau(t)+B_\delta(t)\leq\eta$. If the band crosses the
threshold and $\tau(t)>\eta$, then
$\widehat\tau(t)-B_\delta(t)\leq\eta$ and hence
$\tau(t)\leq\widehat\tau(t)+B_\delta(t)\leq\eta+2B_\delta(t)$.
The case $\tau(t)\leq\eta$ is symmetric. Finally, the conservative rule can
err only by skipping when $\tau(t)>\eta$; the preceding margin argument
bounds the resulting regret by $2B_\delta(t)$.
\end{proof}

\subsection{Fresh-task generalization}
\label{app:fresh-target-generalization}

The local bound concerns one fixed task. This section asks how often it
fails, and how much decision loss remains, when the target is drawn from a
deployment distribution. The first theorem averages over targets; the
neighbor-rate corollary additionally quantifies coverage from a sampled
support bank.

Condition on support-task identities, representations, and repetition counts.
Let $T\sim\TargetDist$ be independent of support outcomes, with weights
$w_i(T)$ measurable from the support features and $T$. The independent
execution assumptions of \cref{thm:local-bound} continue to hold. For a
possibly task-dependent threshold, write
$\widehat a_\eta(t)=\Ind\{\widehat\tau(t)>\eta(t)\}$ and
$a_\eta(t)=\Ind\{\tau(t)>\eta(t)\}$.

\begin{theorem}[Fresh-task gain generalization]
\label{thm:fresh-target-generalization}
Under \cref{eq:effect-smoothness}, for any $\delta\in(0,1)$,
\begin{equation}
  \Pr_{T,\mathrm{exec}}
  \left(
    |\widehat\tau(T)-\tau(T)|>B_\delta(T)
  \right)
  \leq\delta.
  \label{eq:fresh-target-error}
\end{equation}
For any threshold $\eta(t)$ fixed independently of the support execution
outcomes, the plug-in and oracle actions satisfy
\begin{align}
  &\Pr_{T,\mathrm{exec}}
  \left(
    \widehat a_\eta(T)\neq a_\eta(T)
  \right) \nonumber\\
  &\qquad\leq
  \delta+
  \Pr_{T\sim\TargetDist}
  \left(
    |\tau(T)-\eta(T)|\leq B_\delta(T)
  \right).                       \label{eq:fresh-target-action}
\end{align}
For every $b>0$, the last term is at most
\begin{equation}
  \Pr_T(B_\delta(T)>b)
  +\Pr_T(|\tau(T)-\eta(T)|\leq b).
  \label{eq:coverage-margin-split}
\end{equation}
The mean gain error also bounds mean regret:
\begin{align}
  \mathbb E_{T,\mathrm{exec}}R_\eta(\widehat a,T)
  &\leq\mathbb E_{T,\mathrm{exec}}|\widehat\tau(T)-\tau(T)|\nonumber\\
  &\leq
  L\mathbb E_T\rho_w(T)+\varepsilon_\phi
  +\mathbb E_T\sqrt{v_w(T)}.\label{eq:population-regret}
\end{align}
\end{theorem}

\begin{proof}
\textit{Prediction error.}
For every fixed target, \cref{thm:local-bound} bounds the execution failure
probability by $\delta$. Averaging these conditional probabilities over $T$
proves \cref{eq:fresh-target-error}.

\textit{Decision error.}
By \cref{cor:threshold}, disagreement on the
bound's event requires $|\tau(T)-\eta(T)|\leq B_\delta(T)$, proving
\cref{eq:fresh-target-action}. This margin event is contained in
\[
  \{B_\delta(T)>b\}
  \cup
  \{|\tau(T)-\eta(T)|\leq b\},
\]
which gives \cref{eq:coverage-margin-split}.

\textit{Mean regret.}
In \cref{eq:proof-decomp}, independence and $D_{i\ell}\in[-1,1]$ give
$\operatorname{Var}(Z_w(t))\leq v_w(t)$. Jensen's
inequality bounds its expected absolute value by $\sqrt{v_w(t)}$.
Adding the transfer-mismatch bound, averaging over $T$, and applying
\cref{eq:routing-regret-bound} proves \cref{eq:population-regret}.
\end{proof}

\paragraph{Margin consequence.}
If $B_\delta(T)\leq b$ almost surely and
$\Pr_T(|\tau(T)-\eta(T)|\leq u)\leq Cu^\alpha$ for $u\in[0,b]$, with
$C,\alpha>0$, then \cref{eq:fresh-target-action} immediately gives
$\Pr_{T,\mathrm{exec}}(\widehat a_\eta(T)\neq a_\eta(T))
\leq\delta+Cb^\alpha$. This requires few targets near the decision threshold.

\paragraph{Turning support size into a coverage bound.}
As in local $k$-NN regression analysis \citep{kpotufe2011knn}, the next condition
requires source probability mass near every possible target representation.
It allows different source and target distributions, while ruling out targets
with no nearby support. This corollary specializes the weighted bound to
uniform neighbor averaging.

\begin{corollary}[$k$-neighbor generalization rate]
\label{cor:knn-generalization}
Under the assumptions of \cref{thm:local-bound}, let the support
representations $X_i=\phi(t_i)$ be $n$ independent draws from
a source distribution $\SourceDist$, and let $T\sim\TargetDist$ be independent.
Assume the representation space has diameter at most one and that, for every
$z$ in the support of the target representation distribution and every $u\in[0,1]$,
\begin{equation}
  \Pr_{X\sim\SourceDist}(d(X,z)\leq u)\geq c u^\nu
  \label{eq:target-lower-mass}
\end{equation}
for constants $c\in(0,1]$ and a local dimension exponent $\nu>0$.
Give the $k$ nearest support tasks equal
weight and execute each support task for $r$ independent repetitions. If
$u_{n,k}=(2k/(cn))^{1/\nu}\leq1$, then, with probability at least
$1-\delta-e^{-k/4}$ over the support draw, fresh target, and executions,
\begin{equation}
  |\widehat\tau(T)-\tau(T)|
  \leq
  L\left(\frac{2k}{cn}\right)^{1/\nu}
  +\varepsilon_\phi
  +\sqrt{\frac{2\log(2/\delta)}{kr}}.
  \label{eq:knn-generalization-rate}
\end{equation}
\end{corollary}

\begin{proof}
\textit{Coverage.}
Fix a target representation $z$ and let $M_z(u)$ count support points in its
radius-$u$ ball. At $u=u_{n,k}$,
\cref{eq:target-lower-mass} gives
$\mathbb E[M_z(u)]\geq2k$. A multiplicative Chernoff bound therefore yields
\(
  \Pr(M_z(u)<k)\leq e^{-k/4}.
\)
After integrating over the fresh target, the $k$th-neighbor radius is at most
$u_{n,k}$ with the same probability. On this event the uniform-weight radius
$\rho_w(T)$ is at most $u_{n,k}$, while
$v_w(T)=1/(kr)$.

\textit{Gain error.}
Apply \cref{thm:local-bound} conditionally on the support and target.
A union bound combines the coverage-failure probability $e^{-k/4}$ with
the execution-failure probability $\delta$.
\end{proof}

For fixed $\delta$, balancing $(k/n)^{1/\nu}$ and $(kr)^{-1/2}$ gives order
$(nr)^{-1/(\nu+2)}$ for the coverage and noise terms in the error radius when
the optimizing $k$ is feasible.
Representation error remains. The rate requires source mass near target
tasks as in \cref{eq:target-lower-mass}; it does not follow from a
within-inventory audit alone.

\subsection{Evaluating frozen and cross-fitted policies}
\label{app:finite-panel-value}

Action--outcome independence gives an unbiased estimate of policy success
on a fixed set of target tasks. This calculation uses no smoothness or
coverage assumptions.

\begin{proposition}[Frozen finite-panel value]
\label{prop:frozen-value}
Condition on target tasks $t_1,\ldots,t_m$ and on a policy
$\widehat\pi$ fixed independently of their execution outcomes. Then
\begin{equation}
  \widehat V(\widehat\pi)
  =\frac1m\sum_{j=1}^m
    \bigl[\widehat\pi(t_j)Y_{j1}
    +(1-\widehat\pi(t_j))Y_{j0}\bigr]
  \label{eq:frozen-value}
\end{equation}
is unbiased for the policy's mean success on that finite target inventory.
Here $Y_{ja}$ is the observed success, or its repeat mean, in condition $a$.
The analogous estimator for tokens in the selected condition is unbiased
under the same independence condition and integrable token costs.
\end{proposition}

\begin{proof}
Conditioning on tasks and frozen actions $a_j=\widehat\pi(t_j)$,
independence gives $\mathbb E[Y_{ja}\mid t_{1:m},a_{1:m}]=\mu_a(t_j)$
for $a\in\{0,1\}$. Substituting into \cref{eq:frozen-value} and applying
linearity proves the claim. The same argument applies to any integrable
condition-specific cost.
\end{proof}

\paragraph{Leave-one-task-out evaluation.}
Suppose the full predictor configuration is fixed independently of evaluation
outcomes, and task executions are independent conditional on task identities.
An action $a_j$ fitted using only other tasks' outcomes is then independent
of $(Y_{j0},Y_{j1})$, giving the marginal identity
\[
  \mathbb E[a_jY_{j1}+(1-a_j)Y_{j0}\mid t_{1:m}]
  =\mathbb E[a_j\mid t_{1:m}]\mu_1(t_j)
   +(1-\mathbb E[a_j\mid t_{1:m}])\mu_0(t_j).
\]
Summing gives the expected value of the fitted rules. Because actions share
training outcomes, this marginal identity differs from conditioning on the
entire action vector in \cref{prop:frozen-value}. Configuration selection
still requires separate validation. Fixed-prediction bootstrap summaries
measure resampling variability of the fitted results; full-procedure
uncertainty also includes refitting and configuration selection.

\section{Empirical Protocols and Supplementary Results}
\label{app:empirical-details}

The main results use within-inventory, leave-one-task-out $k=6$ prediction.
RQ3 varies the evidence against fixed three-repeat targets; Ridge uses
grouped train/validation/test splits. The prospective panels in
\cref{app:generalization-stress} retain their historical frozen configurations.

\subsection{Execution records and evaluation inventory}

\Cref{sec:experiments} specifies the benchmarks, execution counts, encoders,
and main policies. Each agent has separate outcome and token ledgers.
Transport failures and incomplete attempts are excluded; completed empty
outputs remain failures under the benchmark policy. The retained evidence
comprises paired responses, predictions, embeddings, evaluator outputs,
protocol manifests, and integrity hashes.

The Qwen-Turbo and GLM-5.3-Flash main inventories each contain 4,825 evaluable
tasks and 12,690 condition records. They share 1,536-dimensional question
vectors and the offline gate definition, with stack-specific execution limits
recorded separately, including GLM's ToolQA step cap and low-effort LogicBench
budget. DeepSeek contributes 28,950 condition records over the same 4,825
tasks and uses the shared 1,536-dimensional vectors for the main analysis.
Its separate Ridge controls retain their archived 1,024-dimensional
\texttt{Qwen3-Embedding-0.6B} vectors. Its SpreadsheetBench evaluation retains original and
corrected flags for 18 corrected records and excludes task 42930 because its
test case is absent. Task--skill incidence and bundle definitions appear in
\cref{app:task-skill-structure}.

\subsection{Prediction and baseline controls}

\paragraph{Score support and thresholds.}
\label{app:score-protocol}
All main results use same-family, cosine-weighted $k=6$ neighbors and
the zero-threshold signed-gain rule. Negative cosine weights are clipped to
zero; if all selected weights vanish, the neighbors receive equal weights.
Empty support yields score zero and Skip skill. Each target is excluded from
its support. The with-skill-only and paired scores share exactly the same
neighbors and weights, averaging skill-enabled outcomes or paired gains,
respectively. Repetition counts are specified in
\cref{sec:experiments}.

\paragraph{Configuration provenance.}
\label{app:configuration-origin}
The common $k=6$ rule harmonizes the earlier benchmark-specific analyses
retrospectively, after trajectory collection. Leave-one-task-out exclusion
prevents direct use of a target's outcomes for its score; configuration
selection is outside this evaluation. Keeping the same protocol with $k=3$
yields 61.83\% mean success, a 4.63\% matched-rate advantage, and 26.87\%
execution-token savings; $k=6$ yields 63.12\%, 4.33\%, and 20.83\%, respectively.
These averages weight all 15 panels equally. Both rules have positive
matched-rate point estimates throughout: $k=6$ has higher mean success,
whereas $k=3$ saves more tokens.

\Cref{tab:no-history-controls} collects the outcome-free relevance and
self-judge diagnostics that complement the main prediction and policy
results. \Cref{tab:predictor-baselines} compares estimators given the same
paired history on Qwen-Turbo.

\begin{table}[H]
\centering
\caption{No-history baseline controls. BA gain is the absolute balanced-accuracy advantage of \textsc{SkillDelta} over self-judge.}
\label{tab:no-history-controls}
\scriptsize
\setlength{\tabcolsep}{4pt}
\begin{tabular}{@{}lrrrrl@{}}
\toprule
& \multicolumn{2}{c}{Relevance gate} & \multicolumn{2}{c}{Self-judge} & \\
\cmidrule(lr){2-3}\cmidrule(lr){4-5}
Benchmark & Success (\%) & Saved tokens (\%) & Use (\%) & \shortstack{Router\\tokens/task} & BA gain (\%) [95\% CI] \\
\midrule
\multicolumn{6}{l}{\textit{Qwen-Turbo}} \\
ToolQA & 36.99 & 3.07 & 94.2 & 1164 & +30.41 [27.99, 32.75] \\
MedCalc-Bench & 60.00 & 25.31 & 98.1 & 1583 & +19.65 [17.11, 22.21] \\
BigCodeBench & 50.18 & 46.76 & 99.5 & 4976 & +1.77 [-2.49, 6.09] \\
LogicBench & 83.03 & 35.69 & 98.3 & 1314 & +8.55 [4.78, 12.27] \\
SpreadsheetBench & 18.30 & 35.53 & 93.7 & 2955 & +8.01 [-2.35, 18.68] \\
\midrule
\multicolumn{6}{l}{\textit{GLM-5.3-Flash}} \\
ToolQA & 40.84 & 11.17 & 99.4 & 1322 & +29.86 [26.69, 32.87] \\
MedCalc-Bench & 83.55 & 37.92 & 99.8 & 1691 & +27.19 [23.76, 30.48] \\
BigCodeBench & 61.71 & 25.99 & 98.9 & 5231 & +6.63 [2.60, 10.55] \\
LogicBench & 74.25 & 21.49 & 96.3 & 2224 & +6.60 [3.14, 10.01] \\
SpreadsheetBench & 57.89 & 16.51 & 91.2 & 3264 & +16.19 [11.23, 21.07] \\
\midrule
\multicolumn{6}{l}{\textit{DeepSeek-V4-Flash}} \\
ToolQA & 30.96 & -0.66 & 96.4 & 1396 & +21.54 [19.11, 23.99] \\
MedCalc-Bench & 81.48 & 10.80 & 97.8 & 1680 & +26.75 [24.08, 29.40] \\
BigCodeBench & 59.54 & 39.02 & 96.0 & 5514 & -0.03 [-3.36, 3.21] \\
LogicBench & 80.61 & 24.88 & 91.3 & 1821 & +11.50 [7.02, 15.95] \\
SpreadsheetBench & 51.38 & 20.43 & 44.4 & 3346 & +13.49 [6.57, 20.36] \\
\bottomrule
\end{tabular}
\par\vspace{1mm}
\begin{minipage}{0.98\linewidth}\scriptsize
The relevance gate uses task--skill TF--IDF cosine and an inventory-median threshold. DeepSeek retains its published relevance scores. BA intervals use 10,000 paired task resamples with decisions fixed. Self-judge success and total token costs appear in \cref{tab:rq2-policy-summary}.
\end{minipage}
\end{table}

\begin{table}[H]
\centering
\caption{Paired-evidence estimator controls on the complete Qwen-Turbo inventories.}
\label{tab:predictor-baselines}
\small
\setlength{\tabcolsep}{4pt}
\begin{tabular}{@{}llrrr@{}}
\toprule
Benchmark & Estimator & AUROC & Success (\%) & Tokens saved (\%) \\
\midrule
ToolQA & \textsc{SkillDelta} & 0.879 & 44.55 & 10.37 \\
 & TF--IDF kNN & 0.877 & 44.27 & 10.28 \\
 & Family mean & 0.786 & 45.38 & 7.55 \\
\midrule
MedCalc-Bench & \textsc{SkillDelta} & 0.781 & 70.64 & 14.82 \\
 & TF--IDF kNN & 0.789 & 70.55 & 13.33 \\
 & Family mean & 0.785 & 71.55 & 7.71 \\
\midrule
BigCodeBench & \textsc{SkillDelta} & 0.527 & 47.36 & 49.41 \\
 & TF--IDF kNN & 0.582 & 48.94 & 48.11 \\
 & Family mean & 0.474 & 49.56 & 19.25 \\
\midrule
LogicBench & \textsc{SkillDelta} & 0.621 & 85.22 & 29.05 \\
 & TF--IDF kNN & 0.546 & 85.09 & 27.67 \\
 & Family mean & 0.622 & 87.46 & 16.02 \\
\midrule
SpreadsheetBench & \textsc{SkillDelta} & 0.580 & 17.79 & 36.16 \\
 & TF--IDF kNN & 0.582 & 17.04 & 40.33 \\
 & Family mean & 0.322 & 19.80 & 0.00 \\
\bottomrule
\end{tabular}
\par\vspace{1mm}
\begin{minipage}{0.98\linewidth}\scriptsize
All methods use the same paired cells and zero threshold. TF--IDF kNN replaces the representation while retaining same-family, nonnegative-cosine $k=6$ aggregation; its word unigram/bigram vocabulary uses only inventory text. Family mean averages all other tasks in the group. Savings are relative to Always-on.
\end{minipage}
\end{table}

\subsection{Policy value and uncertainty}

\Cref{tab:main-intervals} reports uncertainty for the main metrics;
\cref{fig:matched-rate-three-stack} compares matched-rate selection value
across agents and benchmarks. Predictions and actions are
held fixed during resampling; \cref{app:finite-panel-value} distinguishes
this analysis from uncertainty for the full fitting procedure.

\begin{table}[H]
\centering
\caption{Task-bootstrap intervals for the main results. Gains are absolute success-rate differences from Always-off. Point estimates appear in \cref{fig:main-results,tab:rq2-policy-summary}. Scores and actions are fixed across 10,000 task resamples.}
\label{tab:main-intervals}
\scriptsize
\setlength{\tabcolsep}{4pt}
\begin{tabular}{@{}llll@{}}
\toprule
Stack & Benchmark & AUROC: 95\% CI & Success gain (\%): 95\% CI \\
\midrule
Qwen-Turbo & ToolQA & [0.849, 0.907] & [13.9, 18.0] \\
Qwen-Turbo & MedCalc-Bench & [0.751, 0.810] & [24.5, 30.4] \\
Qwen-Turbo & BigCodeBench & [0.478, 0.577] & [1.8, 5.5] \\
Qwen-Turbo & LogicBench & [0.574, 0.667] & [7.2, 12.2] \\
Qwen-Turbo & SpreadsheetBench & [0.478, 0.681] & [-0.8, 5.5] \\
\midrule
GLM-5.3-Flash & ToolQA & [0.792, 0.875] & [7.5, 10.7] \\
GLM-5.3-Flash & MedCalc-Bench & [0.819, 0.891] & [7.9, 11.9] \\
GLM-5.3-Flash & BigCodeBench & [0.520, 0.616] & [5.9, 9.7] \\
GLM-5.3-Flash & LogicBench & [0.550, 0.634] & [12.0, 17.4] \\
GLM-5.3-Flash & SpreadsheetBench & [0.638, 0.745] & [21.6, 31.1] \\
\midrule
DeepSeek-V4-Flash & ToolQA & [0.751, 0.811] & [10.8, 14.1] \\
DeepSeek-V4-Flash & MedCalc-Bench & [0.793, 0.852] & [13.8, 17.8] \\
DeepSeek-V4-Flash & BigCodeBench & [0.526, 0.604] & [6.8, 10.1] \\
DeepSeek-V4-Flash & LogicBench & [0.599, 0.693] & [7.1, 11.4] \\
DeepSeek-V4-Flash & SpreadsheetBench & [0.607, 0.721] & [10.2, 16.5] \\
\bottomrule
\end{tabular}
\par\vspace{1mm}
\begin{minipage}{0.98\linewidth}\scriptsize
The gain interval includes zero for Qwen-Turbo SpreadsheetBench. Percentile intervals condition on the fixed scores and actions; they exclude configuration-selection uncertainty.
\end{minipage}
\end{table}

\begin{figure}[H]
  \centering
  \includegraphics[width=\linewidth]{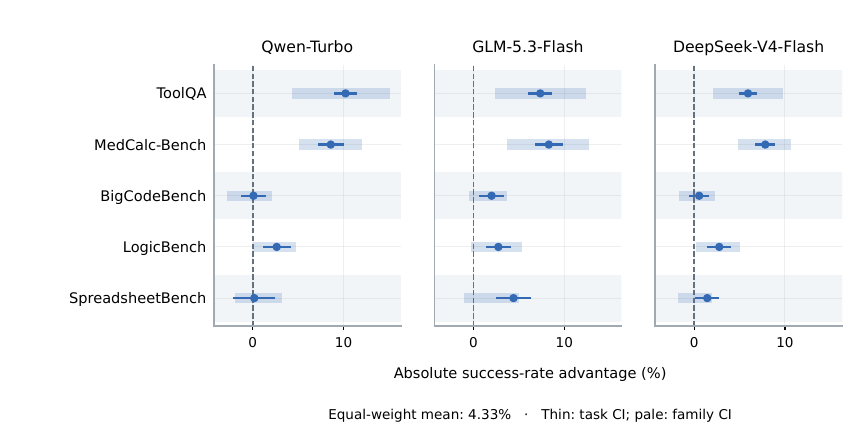}
  \caption{Selection value at the same expected skill-use rate.
  Points show the absolute success-rate advantage over random activation.
  Bars are 95\% intervals from 5,000 task or whole-family resamples, keeping
  actions fixed and recomputing the matching rate. All estimates are
  positive; some intervals include zero.}
  \label{fig:matched-rate-three-stack}
\end{figure}

\paragraph{Within-group selection and between-group allocation.}
\label{app:selection-components}
For a fixed panel, let $a_i$ be its retained action, $p$ its mean use rate, and
$p_{g(i)}$ the mean rate in task $i$'s recorded group. The advantage over
global matched-rate random activation decomposes exactly as
\begin{equation}
  \frac{1}{N}\sum_i(a_i-p)\widehat D_i
  =\underbrace{\frac{1}{N}\sum_i(a_i-p_{g(i)})\widehat D_i}_{\text{within-group selection}}
   +\underbrace{\frac{1}{N}\sum_i(p_{g(i)}-p)\widehat D_i}_{\text{between-group allocation}}.
\end{equation}
The first term compares against random activation at each group's use rate;
the second measures the value of allocating more use to groups with larger
average gains. Groups follow the recorded task structure: skill families for
ToolQA, MedCalc-Bench, and LogicBench, focal groups for BigCodeBench, and task
types for SpreadsheetBench. They need not identify identical skill bundles.

\Cref{fig:selection-components} keeps every main-table action and outcome
unchanged. Across the 15 panels, the equal-weight within-group mean is 0.61\%
[0.37, 0.93], and the between-group mean is 3.72\% [3.47, 3.91], summing to
the 4.33\% advantage. These 95\% intervals use 5,000 bootstrap draws
of tasks within groups, preserving group sizes and sharing task draws across
agents within each benchmark. Matching rates are recomputed in each draw.
The intervals condition on the observed groups and fixed actions; they exclude
predictor refitting and new execution noise.

ToolQA provides the strongest evidence for within-group selection: its
three-agent mean is 4.41\%, and every agent's interval excludes zero. Only
five of the 15 within-group point estimates are positive. BigCodeBench ranges
from $-1.67\%$ to $-0.25\%$; the Qwen-Turbo interval is entirely negative,
while the other two include zero. All remaining within-group intervals
include zero. The positive aggregate advantage thus combines reliable local
selection on ToolQA with substantial allocation value across groups.

\begin{figure}[ht]
  \centering
  \includegraphics[width=\linewidth]{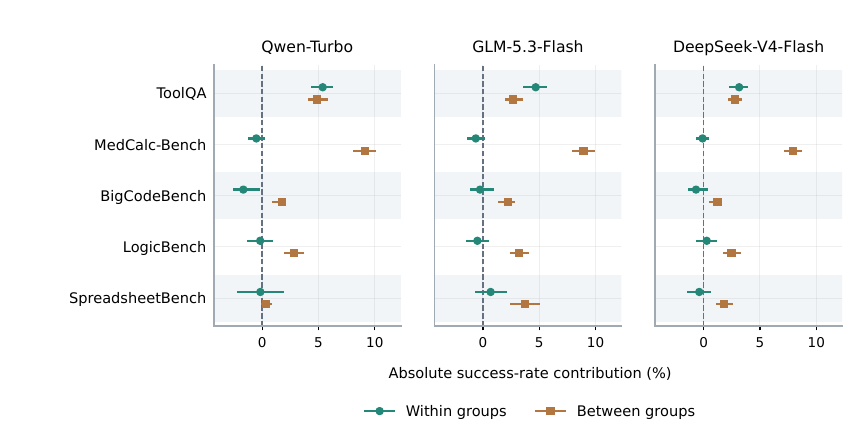}
  \caption{Where selection value comes from.
  Within-group selection and between-group allocation sum to the advantage
  over global matched-rate random activation. Bars show conditional 95\%
  intervals from task resampling within the recorded groups.}
  \label{fig:selection-components}
\end{figure}

\subsection{Single- and dual-Ridge policy controls}

These controls use identical grouped train/validation/test partitions across
the three agents, with a separate outcome ledger for each. Both Ridge variants
share a signed-gain head; Dual Ridge adds an independent incremental-token
head. Five-fold grouped cross-validation selects the Ridge penalties on
training data. Validation selects the least-token policy with success no
lower than Always-on; matched-rule single Ridge uses the same search with
$\lambda=0$. Test outcomes enter only the final evaluation.

\begin{table}[htbp]
  \centering
  \caption{Ridge policy controls on identical fixed test tasks across three agents.
  Acc. and Save are percentages; saving is relative to Always-on.
  $\Delta$ is the absolute accuracy difference.
  Single Ridge uses the same validation search as Dual Ridge with $\lambda=0$.}
  \label{tab:single-dual-ridge}
  \scriptsize
  \setlength{\tabcolsep}{3pt}
  \begin{tabular}{@{}lrr*{2}{rr}l@{}}
    \toprule
    & & & \multicolumn{2}{c}{Single Ridge} & \multicolumn{2}{c}{Dual Ridge} & Dual--On \\
    \cmidrule(lr){4-5}\cmidrule(lr){6-7}
    Benchmark & $N$ & On Acc. & Acc. & Save & Acc. & Save & $\Delta$ (\%) [95\% CI] \\
    \midrule
    \multicolumn{8}{l}{\textit{Qwen-Turbo}} \\
    ToolQA & 286 & 44.76 & 43.71 & 9.09 & 43.71 & 11.27 & $-1.05$ [-3.51,+1.40] \\
    MedCalc-Bench & 220 & 75.91 & 72.27 & 19.17 & 73.64 & 19.55 & $-2.27$ [-5.07,+0.46] \\
    BigCodeBench & 228 & 50.88 & 50.88 & 4.21 & 50.44 & 27.90 & $-0.44$ [-3.07,+2.19] \\
    LogicBench & 152 & 89.04 & 87.94 & 22.83 & 87.94 & 22.83 & $-1.10$ [-4.25,+2.03] \\
    SpreadsheetBench & 80 & 18.75 & 18.75 & 0.00 & 18.75 & 0.00 & $+0.00$ [+0.00,+0.00] \\
    \midrule
    \multicolumn{8}{l}{\textit{GLM-5.3-Flash}} \\
    ToolQA & 286 & 43.71 & 46.15 & 16.26 & 45.45 & 21.98 & $+1.75$ [-1.74,+4.91] \\
    MedCalc-Bench & 220 & 87.27 & 81.82 & 60.33 & 78.64 & 62.44 & $-8.64$ [-14.55,-2.70] \\
    BigCodeBench & 228 & 63.60 & 63.60 & 0.00 & 63.60 & 1.87 & $+0.00$ [-1.32,+1.32] \\
    LogicBench & 152 & 85.75 & 85.09 & 6.24 & 84.43 & 5.08 & $-1.32$ [-2.83,-0.22] \\
    SpreadsheetBench & 80 & 67.50 & 66.25 & 4.04 & 63.75 & 10.24 & $-3.75$ [-8.75,+1.25] \\
    \midrule
    \multicolumn{8}{l}{\textit{DeepSeek-V4-Flash}} \\
    ToolQA & 286 & 35.55 & 35.31 & 0.75 & 35.78 & 4.77 & $+0.23$ [-0.94,+1.52] \\
    MedCalc-Bench & 220 & 90.76 & 92.42 & 8.58 & 92.12 & 10.11 & $+1.36$ [-0.61,+3.67] \\
    BigCodeBench & 228 & 59.50 & 59.50 & 3.87 & 59.36 & 18.66 & $-0.15$ [-1.61,+1.17] \\
    LogicBench & 152 & 87.72 & 86.62 & 6.48 & 87.28 & 12.29 & $-0.44$ [-2.37,+1.40] \\
    SpreadsheetBench & 80 & 57.08 & 55.42 & 16.80 & 55.42 & 16.80 & $-1.67$ [-5.42,+2.92] \\
    \bottomrule
  \end{tabular}
  \par\vspace{1mm}
  \begin{minipage}{0.97\textwidth}\scriptsize
  All stacks reuse the archived task-group partitions and five training folds.
  Both Ridge heads include an unpenalized intercept; penalties minimize training OOF MSE.
  Single Ridge and Dual Ridge minimize validation tokens
  subject to success no lower than Always-on, with explicit Always-on fallback.
  Dual Ridge predicts incremental cost in thousands of tokens and uses $\widehat g-\lambda\widehat c>\eta$.
  The 95\% intervals use 5,000 task-group bootstrap draws with the policy fixed.
  Qwen and GLM use their canonical outcomes and shared 1,536-dimensional question vectors;
  DeepSeek uses three-repeat means and its archived 1,024-dimensional vectors.
  \end{minipage}
\end{table}

The cost head yields additional savings in several panels, with varying test
success. On DeepSeek SpreadsheetBench, matched-rule single and Dual Ridge
select identical test actions, so their saving comes from the shared policy
search. The intervals in \cref{tab:single-dual-ridge} describe the fixed test
policies and exclude refitting and policy-selection uncertainty.
The supplementary reproduction files include the fixed splits, fitted
coefficients, and per-task predictions. An independent offline refit
reproduces the selected penalties and reported policy metrics on all 15 panels.

\subsection{RQ3 sweep settings}
\label{app:rq3-settings}

All sweeps in \cref{fig:rq3-theory-scaling} use fixed three-repeat target
gains, shared question vectors, same-family support, nonnegative cosine
weights, and signed-gain aggregation. Nested support pools are averaged over
five orderings, with $k=\min(n,3)$. The pool and neighborhood sweeps stop at
99, 19, and 39 supports for ToolQA, MedCalc-Bench, and LogicBench, respectively,
retaining every target at every point. The repetition sweep fixes $k=3$ and
independently samples one to three executions without replacement from each
support task's two arms, using 50 draws when fewer than three are retained.
These three-repeat diagnostic targets are distinct from the single-execution
Qwen/GLM ToolQA and MedCalc headline outcomes. When only three-repeat success
counts are retained, arm-wise subsampling uses the corresponding binary
multiset; it does not assume a chronological pairing of repetitions.

\subsection{Supplementary archives and reproduction}

The reproduction guide provides task-level recomputation of main predictions,
policy metrics, intervals, and RQ3 curves from the retained outcomes and
embeddings. \Cref{fig:gain-threshold-tradeoff} enumerates every distinct
threshold-induced policy, including zero and both constant policies;
triangles mark the highest retrospective success. The operating points are checked against
the main table. The package records input hashes, seeds, and commands for
rebuilding the curves and figures. The historical diagnostics include
eligibility counts, family holdouts, risk--coverage curves, local intervals,
and threshold sweeps. Archived catalog-policy comparisons use a separate
calibration regime with a 3\% absolute success-rate margin and remain separate
from the canonical policy estimates.

\section{Task--Skill Structure and Support Coverage}
\label{app:task-skill-structure}

\Cref{tab:benchmark-design} lists the task and skill inventories. All three
agents share the audited task IDs and supplied conditions, so their
task--skill structure is counted once below.

\paragraph{Single-skill families.}
ToolQA has 100 tasks for each of 13 skills and 130 for one; MedCalc-Bench has
20 tasks per skill and LogicBench has 40. Their analysis groups correspond
to the supplied skills. For the within-family predictors, group size limits
the available support.

\paragraph{Bundles and sparse groups.}
BigCodeBench injects all annotated skills together, forming 575 distinct
bundles. Its 3,140 annotation edges give each task two to six supplied skills;
these edges record membership, not separate skill interventions. The 86 focal
analysis groups have median size three, including 29 singletons. Consequently,
a large inventory can still offer little support within a particular group.

This structure helps interpret the selection decomposition in
\cref{app:selection-components}. A singleton has no within-group support
after target exclusion. Its within-group selection term is also exactly
zero, since its action equals the group's use rate; the same holds for
any group whose tasks all receive the same action. Small groups therefore
limit both reusable evidence and opportunities for within-group selection.
Group membership also does not fix the entire injected bundle, so eligible
neighbors can receive different skill combinations. These features are
consistent with the weak within-group results, but the present comparison
does not isolate sparsity from bundle heterogeneity or representation.

\paragraph{One artifact across task types.}
SpreadsheetBench uses one shared skill for 274 cell-level and 125 sheet-level
tasks. The two analysis groups distinguish task types, not different skills.

\paragraph{Observed help and harm.}
As an example, the ToolQA profiles contain 265 helpful and 35 harmful Qwen
pairs, and 177 helpful and 86 harmful GLM pairs; the remainder are ties.
These single-pair observations include execution noise, as discussed in
\cref{app:no-free-lunch}.

\section{Generalization Stress Tests}
\label{app:generalization-stress}

These Qwen-Turbo panels evaluate unseen tasks and shifted task mixtures,
complementing the main retrospective inventories. All predictions were
serialized before the corresponding target outcomes were collected.
They retain their historical configurations, distinct from the current $k=6$,
zero-threshold main analysis.
\Cref{tab:generalization-stress} distinguishes the target constructions and
sequential follow-ups.

\begin{table}[H]
  \centering
  \caption{Prospective generalization stress tests. Success gain is the absolute
  difference between \textsc{SkillDelta} and Always-off; tokens saved are relative to Always-on.
  The panels differ in target construction.}
  \label{tab:generalization-stress}
  \scriptsize
  \setlength{\tabcolsep}{3.5pt}
  \begin{tabular}{@{}lrrrrrr@{}}
    \toprule
    Transfer axis & Known & Target & AUROC & \shortstack{Balanced\\accuracy}
      & \shortstack{Success\\gain} & \shortstack{Tokens\\saved} \\
    \midrule
    ToolQA, unseen tasks within 8 families
      & 96 & 160 & .753 & .689 & +20.00\% & 3.62\% \\
    MedCalc, new tasks across 55 families
      & 440 & 660 & .719 & .667 & +16.21\% & 32.35\% \\
    BigCode, scale follow-up (20/family)
      & 200 & 150 & .716 & .663 & +2.00\% & 64.43\% \\
    BigCode, natural remaining distribution
      & 350 & 200 & .502 & .572 & +7.00\% & 32.88\% \\
    \bottomrule
  \end{tabular}
  \par\vspace{5pt}
  \begin{minipage}{0.96\linewidth}
    \footnotesize
    \textit{Note.} Known is the number of paired support tasks used to freeze
    the predictor; Target is the outcome-unseen audit panel. BigCode scale is a
    sequential follow-up after a 12-task/family pilot; the natural panel is a
    later remaining-population extension.
  \end{minipage}
\end{table}

\paragraph{Frozen predictor settings.}
All four panels use question-only \texttt{text-embedding-3-small} vectors and
nonnegative-cosine-weighted signed gains. ToolQA and MedCalc use same-family
$k=3$ neighbors, with thresholds $0.25$ and $150/440\approx0.341$.
BigCode scale averages all 20 supports per focal family, with threshold $0.105$.
The natural panel uses all 35 supports in an available focal family, otherwise
the global 240 nearest supports, with threshold $-0.0258794$.
The first three thresholds are the support positive-gain prevalences;
the last maximizes leave-one-out balanced accuracy on support tasks, breaking
ties toward fewer skill uses. Each threshold was fixed before target execution.

\paragraph{Matched and broader task mixtures.}
ToolQA retains the support families; its success gain is positive, while
the token-saving interval crosses zero. MedCalc changes the task inventory
and calculator-family mixture under the same agent. Both provide prospective
evidence within supported task mixtures.

\paragraph{Support expansion and distribution shift.}
The BigCode scale panel follows a pilot, so its evidence is sequential.
The natural-distribution extension retains a success gain over Always-off
while task-level ranking falls to chance. This shows why improvement over
skipping and gain ranking must be assessed separately under shift.

The archived ToolQA frozen-action re-execution concerns trajectory stability
on reused tasks and is kept separate from these new-task panels. None of
these panels evaluates unseen skills.

\end{document}